%% file: AAMAS_2026_sample.tex
\documentclass[sigconf]{aamas}

\usepackage{balance} 
\usepackage{algorithm}
\usepackage{algorithmic}
\usepackage{amsmath}

\newtheorem{theorem}{Theorem}

\newtheorem{corollary}[theorem]{Corollary}

\newtheorem{assumption}[theorem]{Assumption}
\newtheorem{problem}[theorem]{Problem}

\setcopyright{ifaamas}
\acmConference[ ]{ }
\copyrightyear{2026}
\acmDOI{}
\acmPrice{}
\acmISBN{}

\acmSubmissionID{1789}

\title[Scenario Generation for Risk-Aware Reinforcement Learning]{Scenario Generation for Risk-Aware Reinforcement Learning with Probably Approximately Safe Guarantees}

\author{Mohit Prashant}
\affiliation{
  \institution{Nanyang Technological University}
  \city{ }
  \country{ }
  }
\email{mohit010@e.ntu.edu.sg}

\author{Arvind Easwaran}
\affiliation{
  \institution{Nanyang Technological University}
  \city{ }
  \country{ }
  }
\email{arvinde@ntu.edu.sg}

\begin{abstract}
Guaranteeing safety is critical to the deployment of reinforcement learning (RL) agents in the real-world, especially as policies learned using deep RL may demonstrate susceptibility to transition perturbations that result in unknown or unsafe behaviour. A method of policy verification is to construct probabilistic barrier-certificates by sampling policy trajectories with respect to safety constraints, thereby demarcating known safe behaviour from unknown behaviour. Obtaining tight upper and lower bounds on the probability of violation of these constraints may be difficult if the policy is susceptible to transition uncertainty or perturbation that places the agent in insufficiently explored states. To address this, we approximate the distribution of the encountered state-space using a variational autoencoder (VAE) and construct upper and lower-bound barrier-certificates using latent characteristics of states to optimize for regions of known, safe behaviour with high confidence. We frame this in our work as a dual optimization problem where the lower-bound barrier-certificate presents a more conservative estimate of the safe region than the upper-bound barrier-certificate. Sampling states that lie within the set difference of the two during training, i.e. the non-robust region, allows us to tighten the upper and lower bounds to provide sharper probabilistic guarantees on safety. Within our study, we describe the guarantees placed and demonstrate the tightness of our bounds experimentally.
\end{abstract}


\newcommand{\BibTeX}{\rm B\kern-.05em{\sc i\kern-.025em b}\kern-.08em\TeX}

\begin{document}


\pagestyle{fancy}
\fancyhead{}


\maketitle 


\input{Sections/1-Introduction}
\input{Sections/2-Related}

\input{Sections/3-Preliminaries}
\input{Sections/4-PAS}

\input{Sections/5-Learning}
\input{Sections/6-Experiments}

\input{Sections/8-Conclusion}



\begin{acks}
This research is supported by the National Research Foundation, Singapore and DSO National Laboratories under the AI Singapore Programme (AISG Award No: AISG2-RP-2020-017). This research is also supported by the National Research Foundation, Prime Minister’s Office, Singapore under its Campus for Research Excellence and Technological Enterprise (CREATE) programme through the programme DesCartes.
\end{acks}


\bibliographystyle{ACM-Reference-Format} 
\balance



\end{document}

%% file: Sections/1-Introduction.tex
\section{Introduction}
Guaranteeing safety is critical to the deployment of deep reinforcement learning (RL) agents in the real-world, especially as policies learnt using RL may fail to generalize to novel scenarios. Specifying safety-constraints for RL agents and subsequently being able to meet them is necessary in domains like aviation, healthcare and robotics. Obstructions to meeting these constraints during the training process include, but are not limited to, biased or insufficient exploration, perturbations to the transition function and differences between the state-distribution encountered during training and deployment \cite{uncer9, uncer12, uncer15, OOD4}. As such, the verification of deep RL policies with respect to safety-constraints is necessary. 

There are numerous works addressing the problem of safety-verification in the field of control, with the classical approach aiming to prove with mathematical rigour that a given trajectory from an initial state will/will not violate the constraints placed \cite{classic1, classic2, classic3}. The limitation of these methods are that they rely on a model-checking principle and require knowledge of the transition dynamics of the system for verification. As such, model-checking methods are difficult to implement if the underlying transition function is unavailable or if the environment is too complex to allow for an approximation of its dynamics.

To address this issue, recent studies have tackled the problem of conducting safety-verification of RL policies in model-free settings, i.e. in the absence of knowledge on transition dynamics \cite{safeRL1, safeRL2}. A method of doing so is by placing probably approximately correct (PAC) guarantees on the policy outcome with respect to the violation of the safety-constraints \cite{pas}. Formulating this problem as a chance constrained program (CCP), where the safety-constraints need only be satisfied with a certain level of confidence, allows for an empirical evaluation of the constraint satisfaction through the sampling of trajectories. In practice, this methodology determines upper and lower bounds on the probability of violation of the safety-constraints, $\epsilon_1, \epsilon_2 \in (0, 1)$, with a user-defined confidence $\delta \in (0, 1)$ \cite{campi}. However, the quality of these guarantees are dependent on the tightness of the bounds, i.e. $|\epsilon_1 - \epsilon_2|$, with looser bounds providing less information on policy safety.

There are various factors that can contribute to inducing a difference between upper and lower-bounds on violation of safety-constraints. This can occur due to a difference between the distribution of states encountered during training and the distribution of states encountered during deployment \cite{uncer9}. Similarly, this can occur due to lacking exploration as the agent is likely to find itself in insufficiently encountered territory without having explicitly violated the safety-constraints \cite{uncer17}. However, this implicitly indicates that the quality of guarantees placed is dependent on the uncertainty of the policy with respect to the state-space.

In complex environments, some states are unlikely to be sufficiently encountered by agents during training and attempting to encounter all states during training can also be infeasible. Therefore, we instead aim to consider the problem of robust exploration, i.e. exploration that builds on an existing demarcated safe region wherein the probability of perturbations to transitions placing the agent in a non-safe region are minimized. We use upper and lower-bound barrier-certificates to determine the region of exploration as the symmetric difference between the sets of safe states corresponding to each barrier, noting that the upper-bound barrier region supersets the lower-bound region. As such, prioritizing exploration within the scenarios corresponding to these states is necessary to narrow the difference between $\epsilon_1$ and $\epsilon_2$.

By formulating barrier-certificates as part of a chance-constrained program (CCP), we take a scenario-based approach to optimize a safe region that we determine with high confidence using a number of sampled trajectories. By constructing a dual optimization problem, with $\epsilon_1, \epsilon_2$ representing the upper and lower-bounds on the violation of the constraints with $\delta$ confidence, we can determine regions where the policy exhibits safe behaviour with the aforementioned degrees of constraint violation. Mapping the barrier-certificates and corresponding safe-regions to a generative distribution trained on the state-space for the environment allows us to sample and generate states within the lower-bound region that are outside of the upper-bound region. With this method, we can conduct a second learning phase for the agent to improve the robustness of its policy by providing states that are tentatively unsafe. Within the rest of this study, we denote the states outside of the upper-bound region and within the lower-bound region as \textit{tentatively unsafe} states.


\subsection{Main Contributions}

Our novel contributions to this body of work is twofold. \textbf{Firstly}, we formulate the safety guarantee as a dual CCP, allowing us to provide both upper and lower-bound barrier-certificates by optimizing safe regions using sample trajectories from the learnt policy. By training a variational autoencoder (VAE) on the encountered state-space we are able to also encode the regions demarcated by the barrier-certificate. Further, encoding barrier-certificates within the latent space of the VAE allows for a determination of the states and characteristics of the safe/unsafe scenarios. \textbf{Secondly}, we utilize the barrier-certificates in conjunction with the VAE to produce states that are relevant to improving the robustness of the safe region and reducing the difference between $\epsilon_1$ and $\epsilon_2$, i.e. tentatively unsafe state. We utilize the latent Gaussian distributive property of the VAE in sampling states between the upper and lower-bound barrier-certificates when constructing novel scenarios. We describe the guarantees placed and demonstrate efficacy of our method through the tightness of our bounds, experimentally, on well-established \textit{Gymnasium} Environments \cite{gym}.

We also demonstrate that improving the quality of the guarantee with strategies like introducing perturbations into transitions or utilizing an epsilon-greedy strategy with the intent of exploration is significantly less effective than generating states within the upper and lower-bound safety regions. We also compare the convergence rates of the bounds produced by our work with other curriculum-based and phased learning methods for RL. 

%% file: Sections/2-Related.tex
\section{Related Work}

\subsection{Safety-Verification of Deep RL Policies}
A body of literature that is related to our work is the safety-verification of deep RL and the provision of guarantees on policy with model-checking using approximations of the environment's transition dynamics that have been learnt using an exploration process \cite{dynamics1, dynamics4, dynamics3}. As opposed to the work on model-checking we cite previously, these studies aim to learn the dynamics with high fidelity and posit approximate guarantees. Similarly, \cite{dynamics5} proposes a program synthesis-based method to derive interpretable and verifiable policies from a pre-trained RL policy, making model-checking easier. However, both of these approaches are likely to have scalability issues in complex domains with large state and action spaces, wherein approximating the dynamics becomes computationally infeasible.

\cite{formal1} introduces a formal verification-based method to ensure safe learning in RL. They provide a formal proof that integrating safety controls obtained through formal verification into the learning process preserves safety guarantees. Similarly, \cite{formal2} propose a probabilistic method for verifying the safety of RL systems. They use a probabilistic model-checking technique on an abstracted Markov Decision Process to calculate an upper-bound on the probability of unsafe behaviour. 
However, like the previous body of literature, these works also depends on learning an accurate RL environment model, which is computationally infeasible to scale with environment complexity.

We briefly discuss a body of research concerning barrier-certificates, which are traditionally used in safety-verification to certify properties of dynamical systems \cite{bc1, bc2, bc3}. \cite{bc4} discusses a safety-verification method that iteratively learns the barrier-certificate alongside the policy and the MDP dynamics. Their approach certifies that the optimized policy within the barrier region is safe without any violations, but it requires a pre-determined set of unsafe states, which may not be available in all problem definitions. Similarly, \cite{classic2} utilizes barrier-certificates to guarantee safety for any policy learned with any RL algorithm using Gaussian Processes to model transition dynamics, which necessitates some domain-specific knowledge. Our work draws most closely from \cite{pas}, which presents probably approximately safe guarantees on safety-constraints being met and constructs barrier-certificate functions over the state-space corresponding with the safe region. We extend this work by mapping the barrier-certificates to generative latent spaces, which we use to generate scenarios that are capable of improving the robustness of the policy through training.


\subsection{Curriculum Learning}
Another body of literature related to our work is curriculum-based reinforcement learning (CRL) \cite{curr1}. (CRL) is a framework to improve the efficiency and performance of reinforcement learning (RL) agents by leveraging structured training regimes. CRL organizes training tasks incrementally, starting from simpler tasks and gradually increasing their complexity. It is generally accepted that there are three components to CRL: \textit{task generation}, the process by which intermediate tasks are created for the agent to learn over; \textit{sequencing}, the process by which the created tasks are ordered; and \textit{transfer learning}, the abstract process under which knowledge is preserved between adjacent domains/environments \cite{curr4, curr5}.

Our study focuses on \textit{task generation} and \textit{sequencing} within the context of safe RL, wherein we attempt to generatively produce scenarios in a sequence that allows for the improvement of probabilistic safety guarantees. Generatively producing scenarios to incrementally improve a policy is a technique that is explored in \cite{curr2}, among others. We use a form of this method that allows us to place probably approximately safe constraints on the latent space to inform the generation process. Similar techniques are explored in \cite{curr3, curr6, curr7}, which attempt to adversarially target the agent by structurally disturbing the input space in order to develop robustness. 




%% file: Sections/3-Preliminaries.tex
\section{Preliminaries and Notation} \label{definitions}

\subsection{Markov Decision Processes}
We model our learning problem as a continuous Markov Decision Process (MDP) characterized by the parameters $(S, A, T, R, \gamma, H)$: $S$ is the state-space within which agents can occupy states; $A$ is the action-space, describing the actions agents can take to transition between states; $T : S \times S \times A \rightarrow [0, 1]$ is a function that determines the transition probability to a state given a state-action; $R : S \times A \rightarrow \mathbb{R}$ is the expected reward from executing a state-action; $\gamma \in (0, 1)$ is the discount factor for future rewards; $H \in \mathbb{Z}^+$ is the length of horizon of each episode. The aim of an RL algorithm applied to this MDP is to learn a state-wise policy that maximizes the expected cumulative discounted reward (value), $V$. 

\footnotesize
\begin{equation*}
     V(s_t) := \max_{a_t \in A} \left[ R(s_t, a_t) + \gamma \sum_{s_{t+1} \in S} V(s_{t+1}) T(s_{t+1} | s_t, a_t) \right]
\end{equation*}
\normalsize

Let $s_t \in S, a_t \in A$ represent the state and action at timestep $t$. Similarly, we denote the policy at timestep $t$ using $\Pi_t : S \rightarrow A$. As the learning problem is model-free, note that $T$ and $R$ are unknown to the learner within this study. Further, a trajectory in an MDP is a sequence of states $\{ s_0, s_1, s_2 ...  \}$, where $s_{n+1}$ is transitioned to from $s_n$ and the probability of observing a particular transition sequence $\tau$ under policy $\Pi_t$ is as follows.

\begin{equation*}
    \mathbb{P}(\tau) = \prod_{i=0}^H T(s_{i+1} | s_i, \Pi_t(s_i))
\end{equation*}

Within our setting, all trajectories must begin from an initial state $s_0 \in S^{init} \subset S$, which is a subset of the state-space. During the first training phase, the sampling of initial states from $S^{init}$ is conducted uniformly.

\subsection{Barrier-Certificates}
Barrier-certificates are primarily utilized in robotics and control systems that have safety-constraints placed on them to evaluate whether a trajectory resulting from a state within the system will lead to a violation of the constraint. As such, they typically take the form of functions, $B : S \rightarrow \mathbb{R}$, denoting whether the policy from a state exhibits safe behaviour or not. In this work, safe states $s \in S$ are characterized by $B(s) \geq 0$ and $B(s') \geq 0$ where $s' \in S$ is the subsequent state transitioned to from $s$ under the current policy. 

Determining an exact barrier-certificate in RL systems is difficult as policies can be stochastic. Noting this, we aim to determine an approximate barrier-certificate, $B(s)$, that identifies states leading to violations of the safety constraints to a degree less than $\epsilon \in (0, 1)$. To do this, we construct probably approximate barrier-certificate (PABC) functions $B_1(s), B_2(s)$ that constitute upper and lower-bounds on $B(s)$ with $\delta \in (0, 1)$ confidence \cite{pas}. With $\epsilon_1 \leq \epsilon_2$, $B_1(s), B_2(s)$ identify the states leading to violations of the safety constraints to a degree less than $\epsilon_1 \in (0, 1)$ and $\epsilon_2 \in (0, 1)$, respectively, with the aforementioned confidence. The regions of safety corresponding with these barrier-certificates are denoted by $S^{B_1} \subset S$ and $S^{B_2} \subset S$, respectively.

In this work, we utilize parametrized linear functions to act as barrier-certificates when identifying safe/unsafe behaviour by sampling trajectories. If a trajectory is characterized by a violation of the safety-constraint that is greater than $\epsilon_1$ or $\epsilon_2$, with high confidence through sufficient sampling, then the states within the trajectory prior to encountering the violation are marked unsafe with a level of probability corresponding to the probability of occurrence of the trajectory by the respective barrier-certificate. We formulate the problem of constructing upper and lower-bound PABC functions in Section \ref{formulation}.

\subsection{Variational Autoencoders} 
Variational autoencoders (VAE) are a class of generative deep neural networks based on the encode-decode approach of autoencoders \cite{VAE}. For an observation, $s$, the model assumes the existence of an underlying latent variable, $z$, that maps to the observation space and is distributed by an intractable prior, $\mathcal{P}(z)$, and a similarly intractable posterior, $\mathcal{Q}(z|s)$. The aim of the encoder and decoder are to compute the posterior and likelihood using Gaussian approximations, $\mathcal{Q}_\theta (z|s)$ and $\mathcal{P}_\theta (s|z)$, respectively, and reconstruct the input by sampling the latent Gaussian distribution. We aim to identify the barrier functions $B_1$ and $B_2$ within the latent space in order to generate the scenarios that will advance the upper-bound barrier-certificate.

\subsection{Setting and Assumptions}
Our work is applicable to both continuous and discrete MDP settings, which we demonstrate with our experiments in Section \ref{experiments}. However, in deriving the theoretical results within our work using probably approximately correct (PAC) guarantees and scenario-based optimization solutions based on \cite{campi}, we include the following assumptions present in the aforementioned works.


\begin{assumption}
    For all states $s_0 \in S^{init} \subset S$ and all values of $\epsilon \in (0, 1)$, $B_1(s_0) \geq 0$ and $B_2(s_0) \geq 0$ with probability $1$, i.e. all initial states exhibit safe behaviour.
\end{assumption}

We take this assumption from \cite{pas}. It is reasonable to make this assumption as barrier-certificates aim to contain the initial states within. Further, our intention in making this assumption is to grow the safe region corresponding with the barrier-certificate outward from the set of initial states. Lastly, if a state in $S^{init}$ is unsafe, it can be removed from the sampling distribution of initial states during training. We also make the following assumption about the upper and lower-bound barrier certificate functions and their corresponding safe regions.

\begin{assumption}
    For parameters $\delta, \epsilon_1, \epsilon_2 \in (0, 1)$, let barrier-certificate functions $B_1, B_2$ tolerate unsafe behaviour violating the constraints to $\epsilon_1, \epsilon_2$ extents, respectively, with $\delta$ confidence. If $\epsilon_1 < \epsilon_2$, then $S^{B_1} \subset S^{B_2} \subset S$.
\end{assumption}

This assumption states that stricter barrier-certificate regions are subsets of less conservative barrier-certificate regions. Considering that both barrier-certificates are constructed simultaneously in our work and the trajectories sampled  during their construction are identical, we find that this is a realistic assumption to make in our work. Lastly, a trivial assumption we make in our work to compute PAC guarantees is that the distribution of sampled trajectories from each initial state $s_0 \in S^{init}$ for a policy $\Pi_t$ during training is representative of the distribution of trajectories sampled during deployment under policy $\Pi_t$.

\subsection{Safety Constraints}
We briefly elaborate on the safety constraints used in this paper. We quantify probabilistic safety of a policy using parameters $\epsilon, \delta \in (0,1)$ with respect to some arbitrary safety-criteria. Most prior works dealing with quantitative safety define unsafe behaviour state-wise, i.e. if the agent enters a particular state, the policy is likely unsafe \cite{pas, safeRL1}. While this serves as a good proxy for measuring unsafe behaviour, it is unable to capture more complex causal relationships, i.e. sequential, that would escape state-wise safety demarcations. An instance of this would be the average speed of a vehicle exceeding a set limit within a defined horizon. Therefore, instead of assigning a safe/unsafe label to states, we assign them to trajectories and partial trajectories via an empirical indicator function. Our method aims to then evaluate probabilistic barrier-certificates over the state-space using the state-wise likelihoods (given safe/unsafe trajectory).

\subsection{Problem Formalization} \label{formulation}
The objective of our work is to guarantee the level of safety of a deep RL policy at a timestep $t$, $\Pi_t$, using barrier-certificates and, subsequently, improve the robustness of the policy using the difference between the upper and lower-bound barrier-certificate regions. Specifically, we construct barrier-certificates over the encoded state-space based on the probability of the agent exhibiting unsafe behaviour via the trajectories that are likely to violate the safety-constraints. 

Assuming that the probability of sampling an initial state $s_0 \in S^{init}$ is $\mathbb{P}_S(s_0)$, denoting a trajectory $\{ s_0, s_1 ... \}$ with $\tau$, the probability of encountering the trajectory is as follows.

\begin{equation*}
    \mathbb{P}(\tau) := \mathbb{P}_S(s_0) \prod_{i=0} T(s_{i+1} | s_i, \Pi_t(s_i))
\end{equation*}

Denoting the set of all possible trajectories using $\mathcal{T}$, we also denote the empirical indicator function $U : \mathcal{T} \rightarrow \{ -1, 1 \} $ which maps safe trajectories to the value $1$ and unsafe trajectories to the value $-1$, denoting a violation of the safety-constraints. Determining a barrier-certificate with absolute certainty is expensive, especially in settings with high complexity as is determining the exact level of violation, $\epsilon \in (0, 1)$, of the safety-constraints by a policy as it requires enumeration over the set of all possible trajectories. Further, as the setting of this problem is model-free, the distribution of trajectories is unavailable. Therefore, we aim to determine probably approximate barrier-certificates that bound the level of violation of the constraints with high confidence by sampling trajectories empirically, introducing a level of confidence $\delta$ in the solution we obtain.

The first part of our work is used to determine upper and lower-bounds on $\epsilon$, with high confidence, and the corresponding safe regions of the encoded state-space for which these bounds hold. We denote the distribution of trajectories at timestep $t$ with $\mathcal{T}_t$. Similarly, we denote the encoded states using $z \in Z \subset \mathbb{R}^n$, for some $n \in \mathbb{Z}^+$, under the learnt probability distribution $z \sim \mathcal{Q}_\theta(z|s)$. Given a trajectory $\tau := \{ s_0, s_1 ...\} \sim \mathcal{T}_t$, we denote distribution of the corresponding encoded trajectory $\zeta := \{ z_i : z_i \sim \mathcal{Q}_\theta(z|s_i) \}$ using the shorthand notation $\zeta \sim \mathcal{Q}_\theta(\zeta|\tau)$ to represent the sampling of an encoded trajectory conditioned on a specific trajectory, $\tau$, and $\zeta \sim \mathcal{Z}_t$ to represent sampling an encoded trajectory from the distribution of all encoded trajectories.

For the remainder of this paper, we also denote barrier functions within the encoded state-space $Z$ using $B_{1\phi}, B_{2\phi}$ and the corresponding mappings of the barrier functions to the state space $S$ using $B_1, B_2$.

\begin{problem} \label{p1}
    For parameter $\delta \in (0, 1)$ and a policy $\Pi_t$, determine a lower-bound value $\epsilon_1 \in (0, 1)$ and barrier-certificate function $B_1$, such that with confidence $\delta$, trajectory $\tau := \{ s_0, s_1 ...\} \sim \mathcal{T}_t$ and corresponding encoded trajectory $\zeta \sim \mathcal{Q}_\theta(\zeta|\tau)$  violate the safety constraints with less than $\epsilon_1$ probability at timestep $t$. That is,

    \begin{equation} \label{b1}
        \mathbb{P} \left( B_{1\phi}(z) > 0 \; | \; \mathbb{P}(U(\tau) < 0) > \epsilon_1 \right) < \delta \;,\; \forall z \in \zeta
    \end{equation}
\end{problem}

Similarly, we formulate the problem of finding the upper-bound.

\begin{problem} \label{p2}   
    For parameter $\delta \in (0, 1)$ and a policy $\Pi_t$, determine an upper-bound value $\epsilon_2 \in (0, 1)$ and barrier-certificate function $B_2$, such that with confidence $\delta$, a trajectory $\tau := \{ s_0, s_1 ...\} \sim \mathcal{T}_t$ and corresponding encoded trajectory $\zeta \sim \mathcal{Q}_\theta(\zeta|\tau)$  violate the safety constraints with less than $\epsilon_2$ probability at timestep $t$. That is,

    \begin{equation} \label{b2}
        \mathbb{P} \left( B_{2\phi}(z) > 0 \; | \; \mathbb{P}(U(\tau) < 0) < \epsilon_2 \right) < \delta \;,\;\; \forall z \in \zeta
    \end{equation}
\end{problem}

We place no assumptions on the distribution of $s_0 \sim S^{init}$ during the training process, letting it be maintained until timestep $t$, when the barrier-certificates are constructed.

%% file: Sections/4-PAS.tex
\section{Probably Approximate Safe Barrier-Certificates} \label{optimization}
Within this section, we aim to solve Problems \ref{p1} and \ref{p2} by formulating them as primal-dual CCP optimizations and sampling trajectories to determine the degree of violation of safety-constraints by the policy. To reiterate, the objective of determining a barrier-certificate is to determine a region within the latent space, with high confidence, from where trajectories are unlikely to violate the safety-constraints established. 


\subsection{Lower-Bound}
In this section, we construct and train a convex barrier-certificate $B_{\phi1}$ using a linear function parametrized using $\phi_1 := \{ \phi_{w1}, \phi_{b1} \}$, where $B_{\phi1}(z) = \phi_{w1} \cdot z + \phi_{b1}$. Similarly, let $\theta$ represent the parameters of the VAE encoding $\mathcal{Q}_\theta(z|s)$. We extend the methodology in prior works by incorporating the techniques for convex scenario optimization discussed in \cite{campi}. 

Let $z \in Z^{B_{\phi1}}$ denote an encoded state lying within the barrier-certificate. Similarly, let $z' \in Z$ denote a state that can be transitioned to from $z$, for $z \sim \mathcal{Q}_\theta(z|s)$ and $z' \sim \mathcal{Q}_\theta(z'|s')$, for a state transition pair $s, s'$ under transition function $T$. 

Per the constraints established in \cite{pas}, a state can only belong to the barrier region $Z^{B_{\phi1}}$ if its subsequent transition also belongs to the region with high probability. Note this assumes the definite existence of states within the barrier-certificate, an assumption satisfied by initializing the barrier-certificate with the encoded set of states belonging to $S^{init}$.

We formulate the primal CCP below by attempting to maximize the minimum barrier score yielded from an encoded state $z'$ that has been transitioned to from $z$ as this is a condition required for the inclusion of $z$ into the barrier region. Note the formulation of the chance-constraints are with respect to the distribution of sampled trajectories $\zeta \sim \mathcal{Z}_t$, and implicitly $\tau \sim \mathcal{T}_t$. As per the definition of a barrier-certificate, we aim to maximize the minimum barrier score of a subsequent state transitioned to for states within the current barrier region, $Z^{B_{\phi1}}$, subject to the constraints posited in Eq. \eqref{b1}.

\begin{align} \label{ccp1}
    & \max_\phi \min_{z \in Z^{B_{\phi1}}, z' \in \zeta} B_{\phi1}(z')  \;, \nonumber \\
    s.t. \;\;\; &  \mathbb{P}(U(\tau) < 0) < \epsilon_1 , \nonumber \\
    & \forall z \in  \zeta \sim \mathcal{Q}_\theta(\zeta|\tau)
\end{align}

As CCP optimization problems are computationally expensive to solve, we utilize the approach of scenario optimization employed in \cite{pas} which requires sampling $N$ trajectories and relaxing the problem to a deterministic version. In practice, the satisfaction of the constraints by a sufficient number of trajectories is also sufficient to verify the quantitative safety of the policy with respect to parameters $\epsilon_1, \delta$. This allows us to construct the following deterministic program.

\begin{align} \label{ccp2}
    & \max_\phi \min_{z \in Z^{B_{\phi1}}, z' \in \zeta} B_{\phi1}(z')  \;, \nonumber \\
    s.t. \;\;\; & U(\tau_i) > 0, \nonumber \\
    & \forall z \in  \zeta \sim \mathcal{Q}_\theta(\zeta|\tau_i), \nonumber \\
    & \forall i \in \{ 1, 2 ... N \}
\end{align}

We note that the formulated optimization may be quite conservative in its solution. Thus, we further utilize the results of \cite{campi}, which provides guarantees on $\epsilon_1$ under a second relaxation of the problem where we permit for the violation of up to $k_1$ convex constraints.

\begin{theorem} \label{theorem}
    \cite{campi} Let $\delta \in (0, 1)$ be a confidence parameter. If $N$ and $k$ are such that

    \begin{equation*}
        \binom{k + |\phi| - 1}{k} \sum_{i=0}^{k+|\phi|-1} \binom{N}{i} \epsilon^i (1-\epsilon)^{N-i} \leq \delta
    \end{equation*}

    \noindent where $|\phi|$ is the number of optimization parameters, then the $N$-fold probability of the event of a constraint violation $V_{|\phi|}$ occurring is bounded as follows $\mathbb{P}^N ( V_{|\phi|} > \epsilon) < \delta$.
\end{theorem}

This theorem places a confidence interval on the expected violation of the constraints based on the number of scenarios observed and the proportion of the scenarios that are in violation of the constraints. We substitute $V_{|\phi|}$ with the constraints in Eq. \eqref{ccp2} in our formulation. A corollary to this theorem is the following.

\begin{corollary} \label{corollary}
    Let $\delta \in (0, 1)$ be a confidence parameter. If $k$ violations of the constraints are observed within $N$ scenarios and $|\phi|$ is the number of optimization parameters, then,


    \begin{equation}\label{k removal}
    \epsilon \ge \frac{1}{N} \left( k + |\phi| - 1 + \sqrt{2k \ln{ \frac{1}{\delta}} + 2(|\phi|-1)\ln{ \frac{1}{\delta}}} \right) 
    \end{equation}
\end{corollary}

\begin{proof}
    The proof of this corollary is included in our supplementary material.
\end{proof}

After evaluating the safety-constraints with respect to $N$ scenarios, substituting the parameters $k_1, \delta$ and $N$ allows us to find a lower-bound on $\epsilon$, i.e. $\epsilon_1$. In practice, this method of constructing a barrier-certificate grows the barrier region outward over the course of the optimization from the set of initial states that are assumed to exhibit safe behaviour.

\subsection{Upper-Bound}

Within this section we aim to formulate the dual optimization by training a convex barrier-certificate $B_{\phi2}$ using a linear function parametrized using $\phi_2 := \{ \phi_{w2}, \phi_{b2} \}$, where $B_{\phi2}(z) = \phi_{w2} \cdot z + \phi_{b2}$. Similar to the lower-bound barrier-certificate, the optimization program for the upper-bound barrier-certificate is defined as follows.

\begin{align} \label{ccp3}
    & \min_\phi \max_{z \in Z^{B_{\phi2}}, z' \in \zeta} B_{\phi2}(z')  \;, \nonumber \\
    s.t. \;\;\; & \mathbb{P}(U(\tau) < 0) > \epsilon_2 , \nonumber \\
    & \forall z \in  \zeta \sim \mathcal{Q}_\theta(\zeta|\tau)
\end{align}

Unlike the objective of the primal problem which is to grow a barrier region outward, the objective of the dual is to shrink $Z^{B_{\phi2}}$ inward to a generous estimate of the safe region containing $z \sim \mathcal{Q}_\theta (z|s)$, $\forall s \in S^{init}$. The deterministic relaxation is as follows.

\begin{align} \label{ccp4}
    & \min_\phi \max_{z \in Z^{B_{\phi2}}, z' \in \zeta} B_{\phi2}(z')  \;, \nonumber \\
    s.t. \;\;\; & U(\tau_i) < 0, \nonumber \\
    & \forall z \in  \zeta \sim \mathcal{Q}_\theta(\zeta|\tau_i), \nonumber \\
    & \forall i \in \{ 1, 2 ... N \}
\end{align}

We denote the number of violations of the constraints in this problem using $k_2$. To avoid bias during the construction of the barrier-certificates, we construct both simultaneously using an identical pool of scenarios. We present the pseudocode for this process in Algorithm \ref{algo1a}.

%% file: Sections/5-Learning.tex
\section{Barrier-Certificate Aided Learning}


Within this section, we discuss how to exploit the barrier-certificates constructed in Section \ref{optimization} in order to reduce the width of the bound on $\epsilon$. With a generative representation of the state-space, i.e. VAE that is trained to reconstruct the state-space, and barrier-certificates $B_{\phi1}(z)$ and $B_{\phi2}(z)$, we are capable of identifying and generating states that are adjacent to the conservative safe-region within the latent space $Z$. The role of the VAE decoder within this process is to map the encoded state within the latent space to the state-space $S$ and reconstruct it.

The solutions to the primal problem are also solutions to the dual problem, and based on their construction within Algorithm \ref{algo1a}, the conservative barrier region is contained within the less conservative barrier region, i.e. $Z^{B_{\phi1}} \subset Z^{B_{\phi2}}$. As our objective is to minimize the set difference between them, it is necessary for the agent to encounter these states.

Note that we use \textit{Phase 1} to describe the agent learning the initial policy and \textit{Phase 2} to describe the process of training the agent using the states generated using the barrier-certificates. Though, even with a generative model, unless transition dynamics are known, extrapolated or the encoding preserves state adjacency, it is impossible to safely train the policy further by attempting to generate subsequent states and trajectories. Therefore, we posit two model-agnostic methods to carry out the second phase of learning.

The first method is to use human/expert guidance. Intuitively, the aim of this process would be to identify states $\{z_1, z_2 ... z_M \}$ lying within $Z^{B_{\phi2}}$ and outside of $Z^{B_{\phi1}}$ in the latent space; decode them using $s_i \sim \mathcal{P}_\theta (s|z_i) \forall i \in \{ 1, 2 ... M\}$; and using an expert, initialize trajectories with the agent within the aforementioned states. This is possible in environments where a highly configurable simulator is utilized. In the absence of the aforementioned simulator, this method is currently impractical in larger, more dynamic environments due to the complexity of interactions. However, the direction of research in \cite{gen1} suggests that it will become practical in the near future.

The second method is described in Algorithm \ref{algoPhase}. The aim of this method is to identify the set of states within $S^{init}$ that are likely to lead to states within $Z^{B_{\phi2}}$ and outside of $Z^{B_{\phi1}}$. Though intuitively less sample-efficient than the first method, this allows for a model-free approach to increasing the robustness of the policy without expert guidance or learning the transition dynamics of the underlying MDP. We demonstrate the results of both methods in Section \ref{experiments}.

Note the sampling procedure in Line 5 of Algorithm \ref{algoPhase}. We construct a set of states belonging to the latent space $Z$ that are within a minimal euclidean distance of states that are known to be unsafe. Our choice of VAEs is necessary for this procedure as proximity of states in the latent space is conditioned on the similarity between latent characteristics of the input, i.e. observations of states belonging to $S$. As such, we can use this encoding property of the VAE to discover new states that are adjacent to and share observable similarities with the unsafe states. 






\begin{algorithm} [t]
\footnotesize
    \caption{Constructing Barrier-Certificates}
    \label{algo1a}
    \textbf{Setup:}
    \begin{algorithmic} 
        \STATE Trained policy $\Pi_t$,
        \STATE Trained VAE encoding distribution $\mathcal{Q}_\theta(z | s)$,
        \STATE Collected $N$ trajectories $\{ \tau_1 ... \tau_N \}$ from policy $\Pi_t$.
    \end{algorithmic}

    \textbf{Procedure:}
    \begin{algorithmic} [1] 
        \STATE Initialize parameters $\phi_1, \phi_2$ such that $Z^{B_{\phi1}} = \{z : z \sim \mathcal{Q}_\theta(z | s), s \in S^{init} \}$ and $Z^{B_{\phi2}} \approx Z$
        \STATE Identify $k_1$ trajectories that violate condition $U(\tau) < 0$
        \STATE Let $S^{N-k_1}$ denote a set of all the states within the $N-k_1$ non-violating trajectories
        
        \REPEAT
        \STATE $Z^{max} := \{\}$
        \FORALL{$ \{ s \in S^{N-k_1} : B_{\phi1}(z) \geq 0, z \sim \mathcal{Q}_\theta(z | s) \}$}

        \STATE $Z^{B_{\phi1}} \leftarrow Z^{B_{\phi1}} \cup \{ z \sim \mathcal{Q}_\theta(z | s)$ \}
        \STATE $S_\tau(s) := \{ s' : (s, s')$ is a transition within the $N-k_1$ non-violating trajectories$ \}$

        \STATE $Z' := \{ z' : z' \sim \mathcal{Q}_\theta(z' | s'), s' \in S_\tau(s) \}$

        \STATE $Z^{max} \leftarrow Z^{max} \cup \arg\max_{z'} ( \{ B_{\phi1}(z') : z' \in Z' \} )$
        
        \ENDFOR

        \STATE $z^{max} := \arg\max_{z'} ( \{ B_{\phi1}(z') : z' \in Z^{max}  \} )$
        \STATE Compute gradient $\nabla_{\phi_1} B_{\phi1}(z^{max})$
        \STATE Update $\phi_1 \leftarrow \phi_1 + \alpha\nabla_{\phi_1} B_{\phi1}(z^{max})$, for some l.r. $\alpha$
        
        \UNTIL{$B_{\phi1}$ converges}

        \STATE Identify $k_2$ trajectories that violate condition $U(\tau) > 0$
        \STATE Let $S^{N-k_2}$ denote a set of all the states within the $N-k_2$ non-violating trajectories
        \STATE Repeat steps 4-15 using $S^{N-k_2}$ and $\phi_2$, minimizing the region $Z^{B_{\phi2}}$. The exact construction is described in the supplementary materials.
    \end{algorithmic}
\end{algorithm}

\begin{algorithm} [t]
\footnotesize
    \caption{Phase 2 Of Learning}
    \label{algoPhase}
    \textbf{Setup:}
    \begin{algorithmic} 
        \STATE Trained VAE encoding distribution $\mathcal{Q}_\theta(z | s)$,
        \STATE Trained VAE decoding distribution $\mathcal{P}_\theta(s | z)$,
        \STATE Collected $N$ trajectories $\{ \tau_1 ... \tau_N \}$ from policy $\Pi_t$,
        \STATE Constructed barrier-certificates $B_{\phi1}(z)$ and $B_{\phi2}(z)$.
    \end{algorithmic}

    \textbf{Procedure:}
    \begin{algorithmic} [1] 
        \STATE Sample encoded trajectories $Z_\tau := \{ \zeta_i : \zeta \sim \mathcal{Q}_\theta(\zeta | \tau_i) \} \}$
        \STATE Construct a set of states $\hat{z} := \{ z  : B_{\phi1}(z) < 0, B_{\phi2}(z) > 0, \text{ for each state within } Z_\tau \}$ 
        \STATE $S^{init'} := \{ \} $
        \FORALL{$z \in \hat{z}$}
            \STATE $\hat{z}' := \{z\} \cup \{ z' \in Z : ||z-z'||_2 < \beta \}$ for some infinitesimal positive value $\beta$
            \STATE $S^N := \{s : s \sim \mathcal{P}_\theta(s | z'), z' \in \hat{z}' \}$
            \STATE $S^{init'} \leftarrow S^{init'} \cup \{ s_0 : s_0, s \in \tau_i, s \in S^N \}$
        \ENDFOR

        \STATE Continue training the agent with initial states sampled uniformly from $S^{init'}$
    \end{algorithmic}
\end{algorithm}

\begin{table*}
    \centering
    \begin{tabular}{|c||c|c|c|c|}
    \hline
            & \multicolumn{2}{|c|}{\textit{Phase 1}} & \multicolumn{2}{|c|}{\textit{Phase 2}} \\ \hline
            & \multicolumn{4}{|c|}{\textbf{Ant}} \\ \hline
            
         \textbf{Algorithm} & $(\epsilon_1, \epsilon_2)$ & Av. Reward & $(\epsilon_1, \epsilon_2)^*$ & Av. Reward\\
         \hline
         \textit{Our Work}* & 0.11, 0.41 & 0.62 & 0.09, 0.18 & 0.94 \\
         \hline
         Vanilla PPO & - & - &  0.20, 0.45 & 0.81 \\
         \hline
         Epsilon Greedy PPO & - & - & 0.22, 0.48 & 0.77 \\
         \hline
         Perturbed Exploration & - & - & 0.12, 0.51 & 0.74  \\
         \hline
         CoDE & - & - & 0.08, 0.29 & 0.93 \\
         \hline
         Genetic Curriculum & - & - & 0.14, 0.45 & 0.90 \\
         \hline
         Curriculum Adv. Training  & - & - & 0.09, 0.24 & 0.96 \\
         \hline

         & \multicolumn{4}{|c|}{\textbf{CartPole}} \\ \hline

         \textit{Our Work}* & 0.21, 0.48 & 0.80 & 0.10, 0.15 & 0.96 \\
         \hline
         Vanilla PPO & - & - &  0.21, 0.43 & 0.91 \\
         \hline
         Epsilon Greedy PPO & - & - & 0.19, 0.52 & 0.83 \\
         \hline
         Perturbed Exploration & - & - & 0.12, 0.51 & 0.62  \\
         \hline
         CoDE & - & - & 0.15, 0.29 & 0.97 \\
         \hline
         Genetic Curriculum & - & - & 0.06, 0.24 & 0.93 \\
         \hline
         Curriculum Adv. Training  & - & - & 0.10, 0.19 & 0.95 \\
         \hline
    \end{tabular}
    \caption{Comparing the PAC Guarantees on Policy Safety Between Algorithms}
    \label{table1}
\end{table*}

%% file: Sections/6-Experiments.tex
\section{Experiments} \label{experiments}


\subsection{Environments and Methodology}

In this section, we describe the experiments conducted in our work and show the practical results of applying Algorithms \ref{algo1a} and \ref{algoPhase} in the training of a policy with safety guarantees. We utilize the \textbf{Ant} and \textbf{CartPole} environments from OpenAI Gym \cite{gym} in order to do so. The choice of \textbf{CartPole} is motivated by its use in classic control and reachability problems, whereas the choice of \textbf{Ant} is motivated by complexity/dimensionality. Within our experiments, we intend to obtain PAC guarantees on policy safety with respect to predetermined safety constraints and demonstrate the efficacy of Algorithm \ref{algoPhase} in tightening the upper and lower bounds on error. 

\textbf{\textit{Our methodology}} consists of \textbf{(1)} determining an unsafe set of states and training a policy using a vanilla RL algorithm, i.e. phase 1; \textbf{(2)} we then determine the PABCs using Algorithm \ref{algo1a} and Corollary \ref{corollary}; \textbf{(3)} we continue training the policy using Algorithm \ref{algoPhase} and redetermine the error-bounds, i.e. phase 2; \textbf{(4)} we compare this with existing methods of targeted learning in terms of the error-bounds derived as well as performance. The specific RL algorithm we utilize in our study to train the policy is the default implementation of \textit{Proximal Policy Optimization} (PPO) \cite{ppo} from the \textit{stable-baselines3} software library \cite{sb3} using the pixel observation of the state. We list hyperparameters for these experiments and initialization seeds in our supplementary material.



Lastly, we comment on the method of training the VAE in our experiments. We utilize a vanilla VAE using the ELBO loss detailed in \cite{VAE}. We weight the KL divergence term in the loss with a coefficient of $0.6$ during the training process. We conduct a batch update to train the VAE following each episode in Phase 1 of learning, where the model aims to reconstruct the pixel values of the encountered states. We compare our methodology of producing sharp guarantees on safety with curriculum-driven methods in Table \ref{table1} using the mean across a range of seeds.

\begin{figure}
  \includegraphics[width=\linewidth]{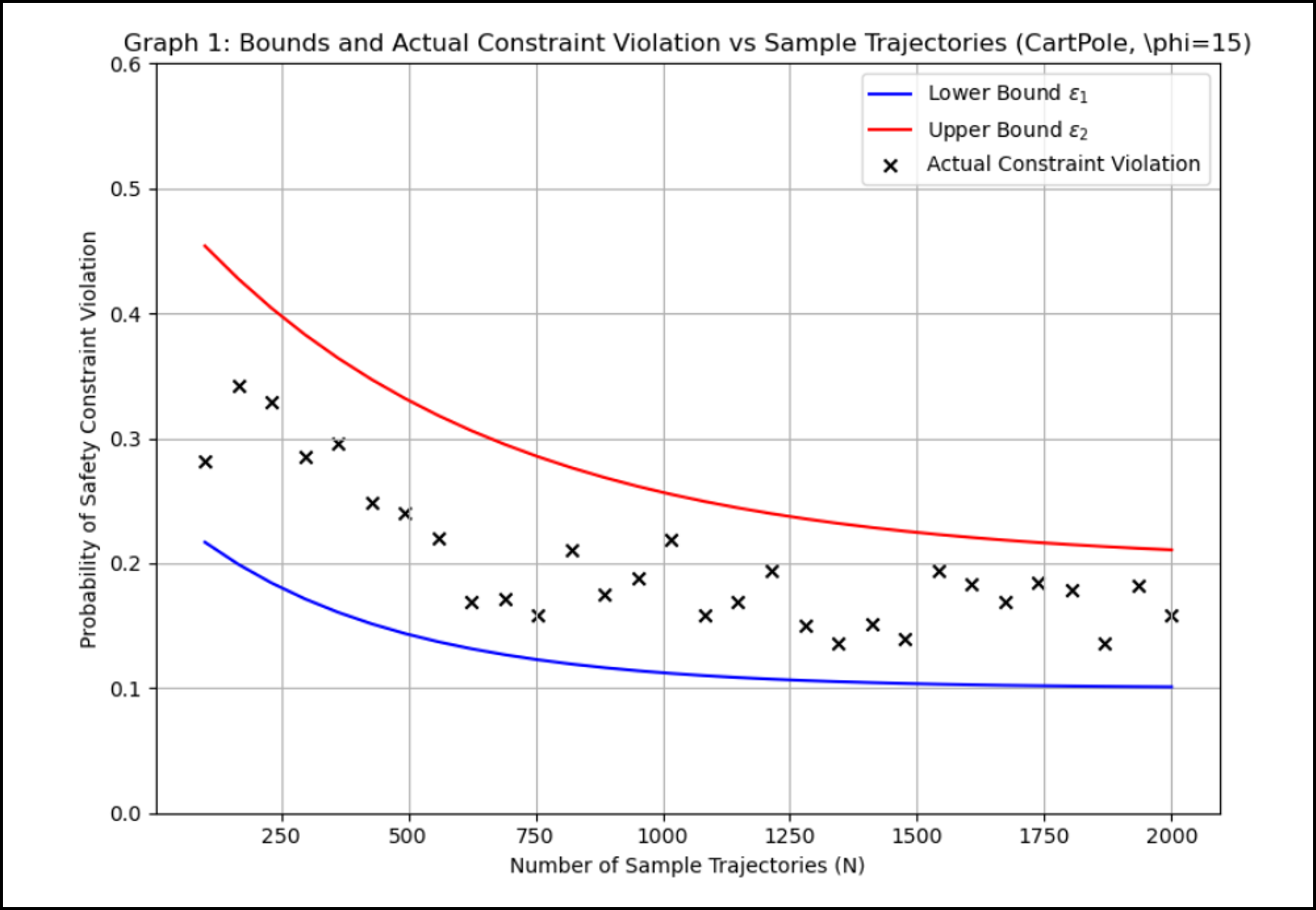}
  \caption{Comparing the error bounds against the actual violation probability per collected samples}
  \label{fig1}
\end{figure}

\subsection{Results}
Per Corollary \ref{corollary}, the values of $\epsilon_1, \epsilon_2$ are dependent on user-selected values $N, \delta$ and the constraint violation parameter $k$. For the results in Table \ref{table1}, $\{ N=1000, \delta=0.001, \beta=0.05, \alpha=0.01 \}$. We elaborate further on hyperparameter choice in our Appendix. To standardise learning, PPO is run for $1e5$ episodes prior to applying a curriculum-based method for another $1e5$ episodes. With Phase 1 carried out using a vanilla PPO algorithm, we compare the results of running Phase 2 of learning using various other learning methods. We include the vanilla PPO as a baseline for Phase 2 alongside the following methods: an epsilon greedy implementation of PPO with $\epsilon=0.2$; an implementation of PPO where Gaussian noise was added to the transition function; CoDE \cite{curr2}; Genetic Curriculum Learning \cite{curr7}; and Curriculum Adversarial Training \cite{curr6}.

The results in Table \ref{table1} show that the learning conducted in Phase 2 by our method results in better quality guarantees across both environments when compared against any of the other methods, with only Curriculum Adversarial Training performing similarly in the CartPole environment. Though, this is not observed in the Ant environment. We hypothesize that this is due to a scaling of the environment complexity (with the complexity of Ant being higher) and our model being able to cope better with more complex transition dynamics. We also note that the performance of our algorithm (recorded using the average reward normalized to [0, 1]) is comparable to the other methods, only being marginally outperformed by by CoDE within the CartPole environment and Curriculum Adversarial Training in the Ant environment. From this, we posit that we are able to achieve a robust learning process without a significant tradeoff in learning efficiency.

We also briefly outline our results in Figure \ref{fig1}, which demonstrate the change in tightness of the bounds with respect to $N$, the number of sample trajectories collected, within the CartPole environment.



%

%% file: Sections/8-Conclusion.tex
\section{Conclusion}

In our work, our technique improves the lower-bound guarantees on safety-constraint violation by an RL policy by generating novel states that are outside the distribution of states encountered during training, but are still between regions of the encoded space proposed by the upper and lower bound barrier-certificates. Beyond demonstrating improved tightness of bounds, our findings emphasize the importance of directly targeting the epistemic uncertainty of policies as opposed to relying on exploratory heuristics. The incorporation of generative modelling in the verification and training process marks a step toward explainable reinforcement learning systems that can provide quantifiable assurances of safety before deployment. In future work, we aim to contrast this method with varied embedding functions and representation in order to observe learning stability. We also aim to extend our work to subtask-specific scenario generation with the intention of applying our method to more complex, hierarchical RL designs that have subtask-specific safety-constraints.